\documentclass{ifacconf}

\usepackage{graphicx}      
\usepackage{natbib}        

\usepackage{amsmath} 
\usepackage{amssymb}  

\usepackage{bm}

\usepackage{algorithm}
\usepackage{algpseudocode}
\usepackage{booktabs}
\usepackage{url}

\usepackage{multirow}  

\newtheorem{definition}{Definition}
\newtheorem{remark}{Remark}

\usepackage{eso-pic}

\newcommand{\cref}[1]{(\ref{#1})}

\usepackage{subcaption}

\begin{document}

\AddToShipoutPictureFG*{%
  \AtPageUpperLeft{%
    \put(50,-100){%
      \fbox{%
        \parbox{0.9\textwidth}{%
          \small
          © 2026 the authors. This work has been accepted to IFAC for publication under a Creative Commons Licence CC-BY-NC-ND
        }%
      }%
    }%
  }%
}

\begin{frontmatter}

\title{Dissipative Latent Residual Physics-Informed Neural Networks for Modeling and Identification of Electromechanical Systems\thanksref{footnoteinfo}} 

\thanks[footnoteinfo]{This work was supported in part by the European Union Horizon Project TORNADO (GA 101189557).}

\author[First]{Youyuan Long} 
\author[First]{Gokhan Solak} 
\author[First]{Arash Ajoudani}

\address[First]{Human-Robot Interfaces and Interaction Lab, Istituto Italiano di Tecnologia, 16163 Genoa, Italy (e-mail: Youyuan.Long@iit.it; Gokhan.Solak@iit.it; Arash.Ajoudani@iit.it)}

\begin{abstract}               
Accurate dynamical modeling is essential for simulation and control of embodied systems, yet first-principle models of electromechanical systems often fail to capture complex dissipative effects such as joint friction, stray losses, and structural damping. While residual-learning physics-informed neural networks (PINNs) can effectively augment imperfect first-principle models with data-driven components, the residual terms are typically implemented as unconstrained multilayer perceptrons (MLPs), which may inadvertently inject artificial energy into the system.
To more faithfully model the dissipative dynamics, we propose DiLaR-PINN, a dissipative latent residual PINN designed to learn unmodeled dissipative effects in a physically consistent manner. Structurally, the residual network operates only on unmeasurable~(latent) state components and is parameterized in a skew–dissipative form that guarantees non-increasing energy for any choice of network parameters. To enable stable and data-efficient training under partial measurability of the state, we further develop a recurrent rollout scheme with a curriculum-based sequence length extension strategy.
We validate DiLaR-PINN on a real-world helicopter system and compare it against four baselines: a pure physical model (without a residual network), an unstructured residual MLP, a DiLaR variant with a soft dissipativity constraint, and a black-box LSTM. The results demonstrate that DiLaR-PINN more accurately captures dissipative effects and achieves superior long-horizon extrapolation performance.
\end{abstract}


\end{frontmatter}

\section{Introduction}
Learning accurate physical models is crucial for controlling embodied systems, especially when model-based control and planning strategies are employed \citep{lutter2019deep}. 
Classical system identification methods typically formulate physical dynamics through parametric models (e.g., state-space representations, ARX/NARX models), then estimate parameters via techniques like least square, prediction error methods (PEM)~\citep{forssell1999closed} or subspace identification \citep{van2012subspace}. While this yields models that are interpretable and often require less data, it also leads to high bias: any structural mismatch (e.g., unmodeled friction, hysteresis, or actuator nonlinearities) cannot be compensated by parameter tuning alone, resulting in systematic prediction errors.

At the other end of the spectrum, neural-network-based black-box identification methods have become an important tool for nonlinear system identification, thanks to their strong nonlinear approximation and generalization capabilities \citep{chen1990non, pillonetto2025deep}. However, this purely data-driven approach heavily relies on the quantity and quality of the available training data \citep{huang2022applications}. In many practical scenarios, where data are scarce or expensive to obtain, most state-of-the-art machine learning techniques often lack robustness and provide no guarantees of convergence in such low-data regimes \citep{raissi2019physics, huang2022applications}.

To alleviate this trade-off between model bias and data efficiency, physics-informed neural networks (PINNs) have emerged as a promising middle ground. PINNs incorporate prior physical knowledge directly into neural network models, mitigating the adverse effects of limited training data, improving generalization, and promoting physically consistent predictions \citep{raissi2019physics}. Subsequent research has demonstrated the broad applicability of PINNs across a wide range of scientific and engineering disciplines, including fluid dynamics, quantum mechanics, and robotics \citep{farea2024understanding}. 

\begin{figure*}[!t]
    \centering
    \includegraphics[width=1.00\linewidth]{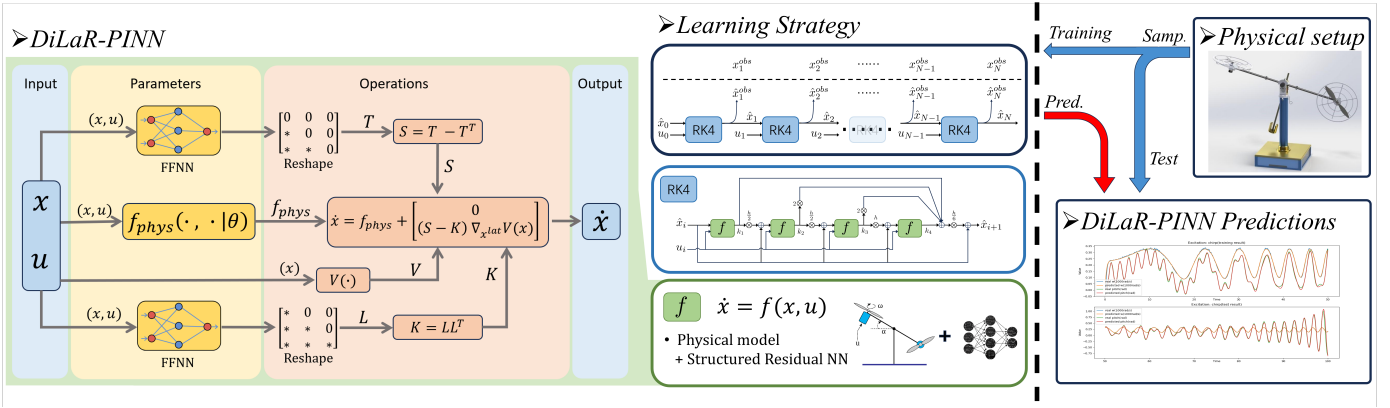}
    \vspace{-0.4cm}
    \caption{\textbf{Overview of the DiLaR-PINN architecture and its learning strategy.} 
    The DiLaR-PINN block presents the model structure, which combines a nominal physical model with a structured dissipative residual network as described in Sec.~\ref{sec:dilar-method}. 
    The Learning Strategy block illustrates how the unknown parameters of DiLaR-PINN are learned from sampled data using recurrent RK4 rollouts, as detailed in Sec.~\ref{sec:rkr-method}.}
    \label{framework}
\end{figure*}

An existing PINN strategy assumes the availability of an approximate but imperfect physical model and then uses a neural network to learn the residual dynamics that the model fails to capture. In this setting, the known physics provides a structured baseline, while the neural component accounts for unmodeled effects such as complex friction, hysteresis, actuator nonlinearities, or unknown external disturbances. 
\cite{rackauckas2020universal} incorporate this residual-learning paradigm into the universal differential equations (UDE) framework, in which neural residuals augment the mechanistic model and are trained jointly from data.
Focusing on mechanical systems, \cite{roehrl2020modeling} propose a physics-informed neural ODE approach based on Lagrangian mechanics. In their method, the nominal equations of motion are derived from kinetic and potential energy, and additional neural networks are used to represent unknown non-conservative forces. 
Although these works show that residual neural networks can effectively augment imperfect physics models and yield data-efficient, accurate, and physically consistent dynamical models, the residual components are implemented as plain multilayer perceptrons without any structural inductive bias, and the identification procedure assumes full-state measurability. 
Building on this line of work, we propose a dissipative latent residual PINN (DiLaR-PINN), as shown in Fig.~\ref{framework}. Through its structural design, the residual network is constrained such that every function in its hypothesis space represents only dissipative dynamics. In addition, we introduce a corresponding learning strategy that enables stable and efficient training even when only a subset of the system states is measurable.

This work focuses on electromechanical systems, where coarse physical models can often be derived from first principles, but the components that are difficult to capture accurately are predominantly dissipative in nature. 
For instance, in robotic manipulators, the Lagrangian formulation yields a standard model structure, yet joint friction—such as Coulomb friction, viscous friction, the Stribeck effect, and stick--slip behavior—remains highly challenging to model analytically \citep{wang2001adaptive,trinh2025accuracy}. 
Similarly, in transformer systems, although the governing equations based on Maxwell’s laws and magnetic circuit theory are well established, eddy-current losses and stray losses are notoriously difficult to characterize from first principles \citep{jahi2021design}. 
In continuum mechanical system such as ropes or continuum robots, even with continuum mechanics and Cosserat rod theory, the most troublesome components arise from internal dissipation, including material damping and inter-fiber friction \citep{lian2019investigation}.
Motivated by these observations, we introduce the DiLaR-PINN, a dissipative latent residual PINN specifically designed to capture these hard-to-model dissipative dynamics. Through its structural design, DiLaR-PINN is guaranteed never to inject energy into the system for any parameter setting, ensuring physically consistent residual modeling.
Furthermore, in this work we restrict the residual network to operate only on the unmeasurable (latent) components of the system state. This design choice is motivated by the fact that measurable states are typically low-order quantities such as positions, whereas the unmeasurable components correspond to higher-order quantities such as velocities. Dissipative dynamics predominantly act on these higher-order states and influence the measurable states only indirectly through integration. By letting the residual network act solely on the latent states, the model can more effectively capture these hard-to-model dissipative effects, as demonstrated in Sec.~\ref{case study}. Nevertheless, this design is not mandatory and can be modified depending on the application.

Overall, the contributions of this work can be summarized as follows:

\begin{enumerate}
    \item We propose \textit{DiLaR-PINN}, a physics-informed residual architecture designed specifically to capture unmodeled dissipative dynamics. Through a structural hard constraint, the residual network is guaranteed never to inject energy into the system.

    \item We introduce a robust learning strategy that enables DiLaR-PINN to be trained efficiently while handling scenarios in which only a subset of the system states is measurable.

    \item We conduct real-world experiments and compare against four baselines—a pure physics model, an unstructured residual MLP, a DiLaR variant with a soft dissipativity constraint, and a black-box LSTM—to demonstrate the effectiveness and superior performance of DiLaR-PINN.

\end{enumerate}

\section{Methodology}

In this section, we begin by formulating the system identification problem under partial measurability (Sec.~\ref{sec:problem}). 
We then present the design of the dissipative latent residual network, which constitutes the core contribution of this work (Sec.~\ref{sec:dilar-method}). 
Finally, we introduce an efficient learning strategy that enables stable and data-efficient estimation of both the physical and residual model parameters (Sec.~\ref{sec:rkr-method}).

\label{section2}
\subsection{Problem Formulation}\label{sec:problem}
Given a state-space equation with unknown parameters,
\begin{equation}
    \label{sse}
    \dot{x} = f(x, u \mid \theta, \phi), \quad \text{with} \quad 
    x = \begin{bmatrix} x^{\text{obs}} \\ x^{\text{lat}} \end{bmatrix},
\end{equation}
where \(x\) represents the system state, consisting of a measurable (observed) part \(x^{\text{obs}}\) 
and an unmeasurable (latent) part \(x^{\text{lat}}\).
The variable \(u\) denotes the system input, while \(\theta\) and \(\phi\) denote the physical 
and neural network (NN) parameters to be identified, respectively.
Our objective is to estimate the unknown parameters \(\theta\) and \(\phi\) using the sampled 
dataset \(\{(x^{\text{obs}}_i, u_i)\}_{i=0}^{N}\), where the sampling interval \(h\) is known.
The parameter \(\theta\) and \(\phi\) can be estimated by minimizing the trajectory prediction error described by the state-space equation~\eqref{sse}.
Accordingly, the optimization problem is formulated as:
\begin{equation}
        \label{opt}
        \begin{aligned} 
        \min_{\theta,\phi,\hat{x}_i} \quad & \sum\limits_{i=1}^{N} \ell(\hat{x}_i^{obs}, x_i^{obs})  \\
            \text{s.t.} \quad \ & \hat{x}_{i+1} = \hat{x}_i + \int_{ih}^{(i+1)h}f(\hat{x}_i, u_i \mid \theta, \phi) \, dt
\quad \forall i \leq N  \\
            & \begin{bmatrix} \hat{x}_i^{obs} \\ \hat{x}_i^{lat} \end{bmatrix}=\hat{x}_i \quad \forall i \leq N, \quad \ \hat{x}_0^{obs}=x_0^{obs},
        \end{aligned}
    \end{equation}
where the $\hat{\cdot}$ notation, refers to a prediction of a variable. The function $\ell$ can be any loss function that describes the difference between the observed and predicted states. 
Directly solving~\eqref{opt} is challenging, since $f(\cdot)$ may be nonlinear and the integral generally has no closed-form solution. 
We next introduce the design of $f$ and the learning strategy used to estimated parameter \(\theta\) and \(\phi\) through~\eqref{opt}.

\subsection{Dissipative Latent Residual Network}\label{sec:dilar-method}
We decompose the system dynamics \( f(\cdot) \) as a nominal physical part augmented with a \emph{latent-only} residual
\begin{equation}
\label{eq:split}
f(x,u\mid\theta, \phi) \;=\;
\underbrace{f_{\mathrm{phys}}(x,u\mid\theta)}_{\text{nominal physical model}}
\;+\;
\underbrace{\begin{bmatrix} 0 \\ r_\phi(x,u) \end{bmatrix}}_{\text{latent residual}} ,
\end{equation}
where \( f_{\mathrm{phys}}(x, u \mid \theta) \) denotes the known structural component of the dynamics, 
which is physics-informed and typically derived from first-principle modeling, with unknown physical parameters \(\theta\), while \( r_\phi(x, u) \) serves as a compensatory term capturing the unmodeled physical dynamics. Furthermore, we aim for \( r_\phi(x, u) \) to learn a dissipative representation, meaning that it should not actively inject energy into the system.

To formalize this dissipativity requirement, consider an ``energy'' function \(V:\mathbb{R}^{n_x}\!\to\mathbb{R}_{\ge 0}\) that is continuously differentiable on the \emph{full} state.
Its time derivative satisfies
\begin{equation*}
\dot V(x)\;=\;\nabla_x V(x)^\top f_{\mathrm{phys}}(x,u\mid \theta)\;+\;
\underbrace{\nabla_x V(x)^\top \begin{bmatrix} 0 \\ r_\phi(x,u) \end{bmatrix}}_{\le 0},
\end{equation*}
which leads directly to the following dissipativity condition:
\begin{equation}
\label{dissipation-req}
\nabla_{x^{\mathrm{lat}}} V(x)^\top r_\phi(x,u) \;\le 0 ,
\quad \forall\, (x,u).
\end{equation}

\begin{prop}[Dissipation Guarantee]\label{prop 1}
Suppose the latent residual network is parameterized as
\begin{equation}
\label{param}
r_\phi(x,u) \;=\; \big(S_\phi(x,u) - K_\phi(x,u)\big)\,\nabla_{x^{\mathrm{lat}}} V(x),
\end{equation}
where $S_\phi(x,u)$ is skew-symmetric ($S_\phi^\top = -S_\phi$) and 
$K_\phi(x,u)$ is positive semidefinite ($K_\phi \succeq 0$).
Then, $r_{\phi}(x,u)$ satisfies~\cref{dissipation-req}.
\end{prop}

\begin{pf}
By definition,
\begin{align*}
\nabla_{x^{\mathrm{lat}}} V(x)^\top r_\phi
&= \nabla_{x^{\mathrm{lat}}} V(x)^\top (S_\phi - K_\phi)\,\nabla_{x^{\mathrm{lat}}} V(x) \\[1pt]
& \hspace{-55pt} = \nabla_{x^{\mathrm{lat}}} V(x)^\top S_\phi \,\nabla_{x^{\mathrm{lat}}} V(x)
 - \nabla_{x^{\mathrm{lat}}} V(x)^\top K_\phi \,\nabla_{x^{\mathrm{lat}}} V(x).
\end{align*}
Since $S_\phi^\top = -S_\phi$, we have $\nabla_{x^{\mathrm{lat}}} V(x)^\top S_\phi \nabla_{x^{\mathrm{lat}}} V(x) = 0$ for any vector $\nabla_{x^{\mathrm{lat}}} V(x)$.
Moreover, $K_\phi \succeq 0$ implies $\nabla_{x^{\mathrm{lat}}} V(x)^\top K_\phi \nabla_{x^{\mathrm{lat}}} V(x) \ge 0$. \\
Therefore, $\nabla_{x^{\mathrm{lat}}} V(x)^\top r_\phi(x,u) \le 0$, as claimed.
\end{pf}


Proposition~\ref{prop 1} provides the theoretical foundation for the structural design of $r_\phi$, where $S_\phi$ describes the \emph{conservative energy flow} and $K_\phi$ describes the \emph{energy loss}.
For the positive semidefinite (PSD) matrix $K_\phi$, one can employ a
Cholesky-like parameterization ($K_\phi = L_\phi L_\phi^\top$, where $L_\phi$
is lower triangular). An ensuing question, however, is whether this parameterization
in~\cref{param} covers all $r_\phi$ satisfying~\cref{dissipation-req} for any
given $(x,u)$, as we do not wish to lose expressivity. The following proposition
shows that it does.

\begin{definition}
Define the \emph{dissipative cone}
\[
\mathcal{C}(x)\;:=\;\{\, r\in\mathbb{R}^{n_{{\mathrm{lat}}}}\;:\; \nabla_{x^{\mathrm{lat}}}V(x)^\top r \le 0 \,\}.
\]
\end{definition}

\begin{prop}[Coverage of the Dissipative Cone]
\label{prop 2}
Let \\ $g(x):=\nabla_{x^{\mathrm{lat}}}V(x)\in\mathbb{R}^{n_{{\mathrm{lat}}}}$ and define the \emph{skew--dissipative representable set}
\[
\mathcal{H}(x):=\big\{\, (S-K)\,g(x)\;:\; S^\top=-S,\ \ K\succeq 0 \,\big\}.
\]
If $g(x)\neq 0$, then $\mathcal{H}(x)=\mathcal{C}(x)$.
\end{prop}

\begin{pf}(Two inclusions)
($\subseteq$) For any $S^\top=-S$, $K\succeq 0$,
\[
g^\top (S-K)g = g^\top S g - g^\top K g \le 0,
\]
hence $(S-K)g\in\mathcal{C}(x)$.

($\supseteq$) Given any $r\in\mathcal{C}(x)$ with $g\neq 0$, set
\[
\beta := \frac{g^\top r}{\|g\|^2}\ (\le 0),\quad
r_\parallel := \beta g,\quad r_\perp := r-r_\parallel  \ (\text{so } g^\top r_\perp=0).
\]
Let $\gamma:=-\beta\ge 0$ and define
\[
K := \gamma\,\frac{g\,g^\top}{\|g\|^2}\ \ (\succeq 0),\qquad
S := \frac{r_\perp\,g^\top - g\,r_\perp^\top}{\|g\|^2}\ \ (S^\top=-S).
\]
Then $K g=\gamma g=-\beta g$ and $S g=r_\perp$, hence $(S-K)g=r_\perp+\beta g=r$. 
Thus $\mathcal{C}(x)\subseteq\mathcal{H}(x)$.
\end{pf}

With Propositions~\ref{prop 1} and \ref{prop 2}, the latent residual parameterization in \cref{param} is both \emph{energy-safe} and \emph{pointwise expressively complete}. Furthermore, the dissipative latent residual \emph{preserves input-to-state stability (ISS)}: as established in Propositions~\ref{prop 3}, if the nominal physical model is ISS, augmenting it with a \emph{dissipative} latent residual network leaves the overall system \emph{ISS}. This aligns with intuition, since the added network can only \emph{dissipate}—rather than inject—energy.

\begin{prop}[ISS Preservation]\label{prop 3}
Assume the nominal \\ physical model $\dot x \;=\; f_{\mathrm{phys}}(x,u\mid\theta)$ admits an ISS--Lyapunov function $V:\mathbb{R}^{n_x}\!\to\mathbb{R}_{\ge0}$ with $V\in C^1$ s.t.
\begin{subequations}\label{eq:iss-phys}
\begin{align}
\alpha_1(\|x\|)\le V(x)\le \alpha_2(\|x\|),
\label{eq:iss-phys-a}
\\
\nabla_x V(x)^\top f_{\mathrm{phys}}(x,u\mid\theta)\le -\alpha_3(\|x\|)+\sigma(\|u\|),
\label{eq:iss-phys-b}
\end{align}
\end{subequations}
for some class-$\mathcal K_\infty$ functions $\alpha_i$ $(i=1,2,3)$ and $\sigma$.
Consider the augmented system \cref{eq:split} with latent residual parameterization \cref{param}.
Then the augmented system is ISS with respect to $u$ under the same Lyapunov function $V$.
\end{prop}

\begin{pf}
Along trajectories of system \cref{eq:split},
\[
\dot V(x,u)\ =\ \nabla_x V(x)^\top f_{\mathrm{phys}}(x,u\mid\theta)\ +\ \nabla_{x^{\mathrm{lat}}}V(x)^\top r_\phi(x,u).
\]
By Proposition~\ref{prop 1}, the residual satisfies \\ $\nabla_{x^{\mathrm{lat}}}V(x)^\top r_\phi(x,u)\le 0$ for all $(x,u)$.
Combining with \cref{eq:iss-phys-b} yields
\[
\dot V \le -\alpha_3(\|x\|)+\sigma(\|u\|),
\]
which is the ISS Lyapunov inequality; hence the augmented system \cref{eq:split} is ISS.
\end{pf}

\begin{remark}\textit{(Motivation for a Latent-Only, Dissipative Residual).}
A latent-only residual prevents the network from overwriting $f_{\mathrm{phys}}$ with a purely data-fitting surrogate on the measured outputs. 
Moreover, in electromechanical systems, dominant modeling errors typically arise from dissipative, often latent effects (e.g., complex friction/damping, hysteresis, fluid–structure interactions) that rarely manifest directly in the measurable states (e.g., end-effector position, body pose)~\citep{sciberras2023thermo,mahmoudkhani2024new,jayawardhana2011dissipativity}. 
Therefore, we design $r_\phi$ to be \emph{latent-only and dissipative}, capturing unmodeled latent dynamics without injecting energy. 
This restriction, however, is not absolute—when the measurable-state dynamics themselves are imperfectly modeled, the residual can be extended accordingly.
\end{remark}

\begin{remark}[\textit{Design of Energy Function $V$}]
In this work, \( V(x) \) denotes a \emph{generalized energy function}, which may represent either a physical energy derived from first principles or a task-specific Lyapunov energy measuring deviation from a desired latent equilibrium. 
For electromechanical systems (our primary focus), the state variables often possess clear physical meanings (e.g., positions, velocities, and angles), and the corresponding energy expressions can often be derived directly from first principles, as demonstrated in our case studies (Sec.~\ref{case study}).
\end{remark}

\begin{remark}[\textit{Connection to Port-Hamiltonian Systems}]
\mbox{}\\
The port-Hamiltonian system (PHS) is typically expressed as 
\[
    \dot{x} = \big[J(x) - R(x)\big]\nabla_x H(x) + G(x)u(t),
\]
where \(J=-J^\top\) and \(R\succeq0\). 
The PHS formulation originates from the principle of \emph{energy flow and power conservation}, describing how energy is exchanged among subsystems through physical ports~\citep{duindam2009modeling}.
In contrast, our proposed dissipative residual structure in \cref{param} is obtained by explicitly \emph{enforcing non-increasing energy}.
Despite the distinct motivations, both formulations share a structurally analogous form.
Therefore, Proposition~\ref{prop 2} can also be interpreted as a theoretical complement to the PHS framework, 
showing that the PHS-type structure \([J(x)-R(x)]\nabla_x H(x)\) can represent all vector fields that ensure non-increasing internal energy when the input is zero.
\end{remark}

\subsection{Learning Strategy}\label{sec:rkr-method}
Since the nonlinear differential equation $f(\cdot)$ in \cref{opt} generally does not admit a closed-form integral, the fourth-order Runge–Kutta (RK4) method is adopted for numerical integration.
The structure of the RK4 module is shown in Fig.~\ref{RK4}. Its inputs are the current system state $\hat{x}_i$ and the control input $u_i$. The sampling interval $h$ is a fixed step size of the integrator. The output is the predicted next-step state, $\hat{x}_{i+1}$.

\begin{figure}[h]
      \centering
      \includegraphics[width=\linewidth]{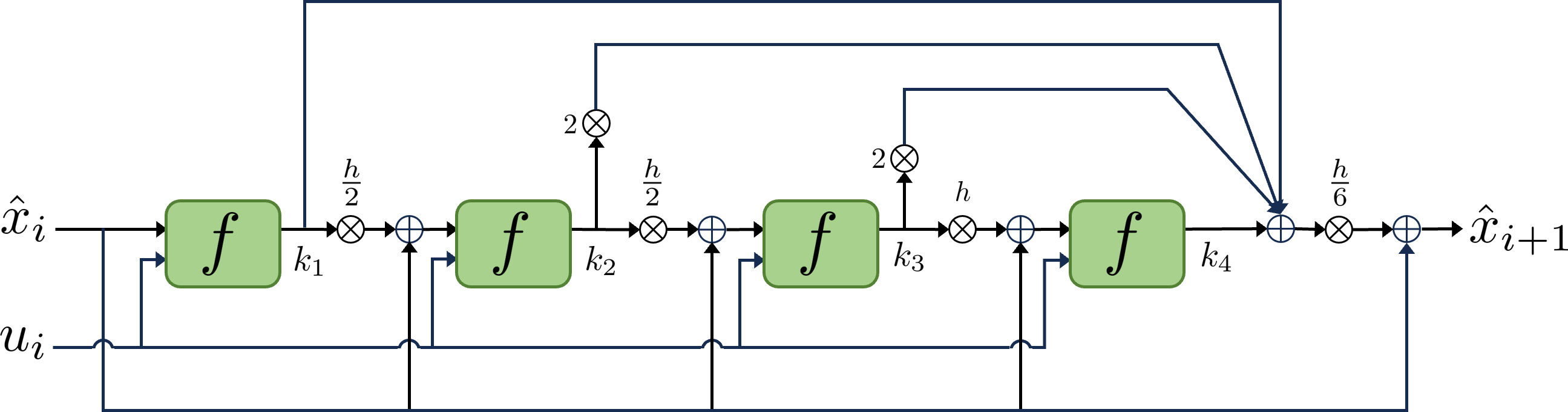}
      \caption{Schematic of the fourth-order Runge-Kutta (RK4) method.}
      \label{RK4}
\end{figure}

The learning strategy adopts a recurrent structure as its backbone, not only because it conforms to the dynamic characteristics of state-space systems, but more importantly because the system states contain unmeasurable latent variables, and the recurrent structure helps capture these latent state variables. 
The recurrent rollout used in the learning process is illustrated in Fig.~\ref{RKR-PINN}. Based on estimates of the initial latent state $\hat{x}^{lat}_0$, \(\theta\) and \(\phi\), it recurrently predicts the state time series using input data $u_i$ and compares the measurable portion with the actual measurement data. 
\begin{figure}[h]
      \centering
      \includegraphics[width=\linewidth]{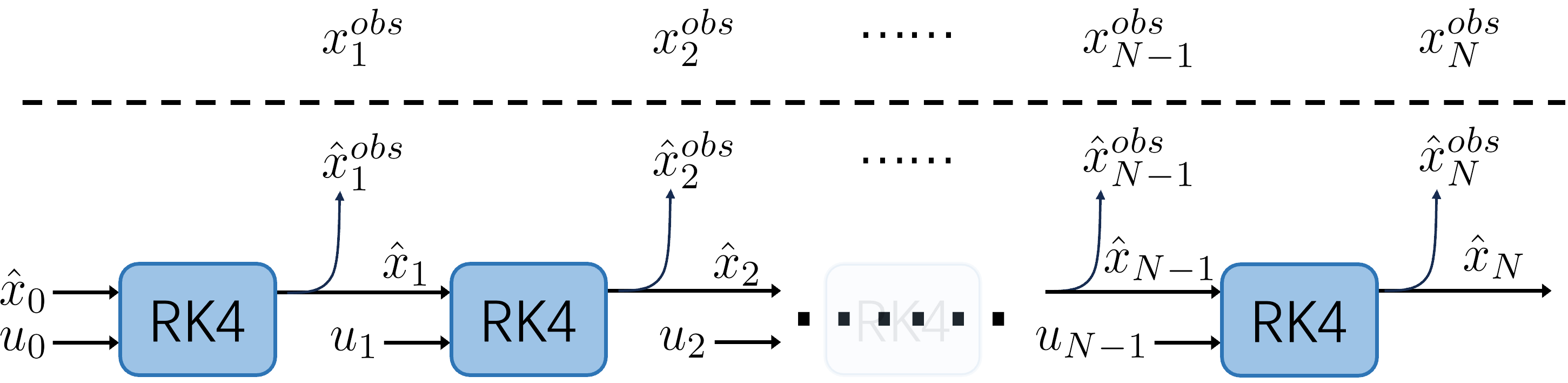}
      \caption{The recurrent RK4 rollout structure for model training.  Please note that $\hat{x}_i$ is composed of $\hat{x}_i^{obs}$ and $\hat{x}_i^{lat}$.}
      \label{RKR-PINN}
\end{figure}

To compare the differences between the predicted outputs and the measured data, a loss function $\ell$ should be defined. A commonly used loss function is the mean squared error (MSE); however,  directly using it may not lead to good parameter estimation results. This is because, due to the choice of units, the magnitudes of the responses for each state can vary significantly. 
Consequently, when using MSE, states with smaller responses may be overlooked since their loss contributions are comparatively minor \citep{yang2023physics}. To cope with this issue, we propose using variance, given its effectiveness in measuring data dispersion, to define the corrected MSE (CMSE):
\begin{equation}
    \label{cmse}
    \begin{aligned}
    \text{CMSE}(\hat{x}_i^{obs}, x_i^{obs}):=\sum_{j=1}^{m}w_j( \hat{x}_{i,j}^{obs}-x_{i,j}^{obs})^2, \\ 
    \text{where} \quad w_j = \frac{\sum_{i=1}^{m} 
\mathrm{Var}\!\left(\{x^{\mathrm{obs}}_{i,j}\}_{i=0}^{N}\right)}
{\mathrm{Var}\!\left(\{x^{\mathrm{obs}}_{i,j}\}_{i=0}^{N}\right)},
    \end{aligned}   
\end{equation}
here $m$ represents the dimension of vector $x^{obs}_{i}$ and $x^{obs}_{i,j}$ represents the \(j\)-th component of $x^{obs}_{i}$. 

Finally, to enhance training efficacy, we propose a curriculum learning strategy with progressive sequence length extension. When processing excessively long time series, the increased recurrence iterations amplify input-output nonlinearity in the neural architecture, leading to highly non-convex loss landscapes that frequently trap optimization in suboptimal local minima. Conversely, insufficient training sequences risk overfitting due to limited pattern diversity. To mitigate these issues, our strategy initiates training with truncated subsequences and gradually expands the temporal horizon to encompass full-length sequences as model competence improves. This phased approach allows the network to first master local temporal dependencies through `easier' truncated learning objectives before tackling global sequence modeling challenges.
The learning strategy is summarized in Algorithm \ref{alg}.
\begin{algorithm}
    \caption{Learning Strategy for DiLaR-PINN}
    \textbf{Input:} Dataset $\{(x^{obs}_i,u_i)\}_{i=0}^{N}$, initial sequence length $l$, sequence length increment $\triangle l$, sampling interval $h$, loss threshold $loss\_thr$
    \begin{algorithmic}[1] 
        \State Initialize vectors $\hat{x}_0^{lat}, \theta,$ and $\phi$
        \For{$t = 1, \ldots, T$}
            \State $\hat{x}_0 \leftarrow (x^{obs}_0,\hat{x}^{lat}_0)$
            \For{$i = 0, \ldots, l$}
                \State $\hat{x}_{i+1}=RK4(\hat{x}_i,u_i \mid \theta, \phi, h)$
                \State $(\hat{x}^{obs}_{i+1},\hat{x}^{lat}_{i+1})\leftarrow\hat{x}_{i+1}$
            \EndFor
            \State $loss=\sum_{i=1}^{l}CMSE(\hat{x}_i^{obs}, x_i^{obs})$
            \State Update $\hat{x}_0^{lat}, \theta,$ and $\phi$ based on $loss$
            \If{$loss < loss\_thr$ and $l < N$} 
                \State $l \leftarrow l + \triangle l$
            \EndIf
        \EndFor
    \end{algorithmic}
    \label{alg}
\end{algorithm}

\section{Case Study: Helicopter System}
\label{case study}
To validate our method, we consider a representative electromechanical system—the helicopter setup—as the identification benchmark. Several existing baseline methods are included for comparison. The objective of this case study is twofold: (1) to demonstrate that our approach can successfully learn the unmodeled dissipative dynamics, and (2) to evaluate whether it provides performance advantages over the existing baselines.

\subsection{Comparison Baselines}
To validate the major design decisions and their impacts on performance, we compare with \textbf{four} different alternative approaches:

\begin{itemize}

\item \textit{Nominal physical model (NPM).}  
The coarse physics-only model $\dot x = f_{\mathrm{phys}}(x,u\mid\theta)$ serves as the first baseline, 
against which the performance improvement brought by the proposed DiLaR-PINN can be evaluated.

\item \textit{Universal Differential Equation (UDE).}  
Following the UDE framework~\citep{rackauckas2020universal},
a standard unstructured fully connected feedforward neural network (FFNN) is added as an unconstrained residual
to the \emph{entire} physical model, e.g.  
$\dot{x} = f_{\mathrm{phys}}(x,u) + f_{\mathrm{NN}}(x,u)$.

\item \textit{DiLaR-Soft.}  
This baseline borrows the idea of the pointwise Lyapunov loss proposed by~\cite{rodriguez2022lyanet},
which enforces exponential convergence to labeled equilibria through a potential function~$V$.  
Here, the same principle is applied to the latent dissipativity condition by introducing
a soft penalty term.  
The overall training loss is formulated as
\[
\begin{aligned}
\mathcal{L} 
= & \sum_{i=1}^{N} \ell(\hat{x}_i^{\mathrm{obs}}, x_i^{\mathrm{obs}}) \\
  & + \lambda_{\mathrm{diss}}
      \sum_{i=1}^{N} 
      \operatorname{ReLU}\!\big(
      \nabla_{x^{\mathrm{lat}}} V(x_i)^\top r_\phi(x_i,u_i)
      \big),
\end{aligned}
\]
where the second term penalizes any violation of the dissipativity condition,
encouraging the network to approximately maintain energy dissipation during training.

\item \textit{LSTM.}  
A purely black-box baseline.

\end{itemize}

\subsection{System Description}
The helicopter setup is shown in Fig.~\ref{heli setup}, which consists of a beam attached to a fixed pole with DC motors equipped with propellers mounted at both ends. The system’s input is the left motor's voltage \( u \) (V), and its outputs are the measurements from two sensors: the propeller rotational speed \( \omega \) (1000 rad/s) and the helicopter pitch angle \( \alpha \) (rad). The right motor is not controlled.

\begin{figure}[h]
  \centering
  \begin{subfigure}[b]{0.37\linewidth}
    \centering
    \includegraphics[width=\linewidth]{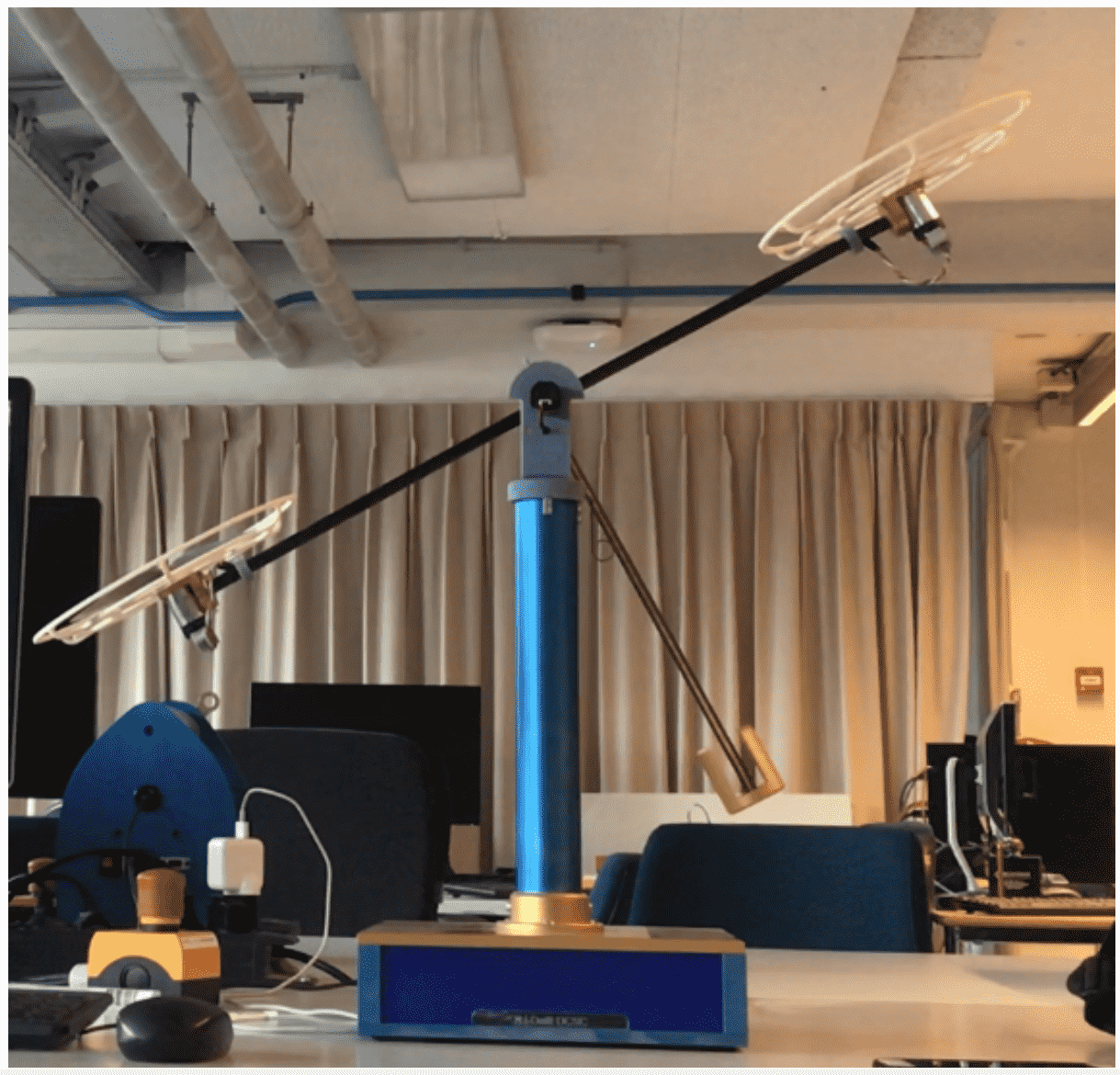}
    \caption{}
    \label{fig:heli-sub1}
  \end{subfigure}
  \hfill
  \begin{subfigure}[b]{0.45\linewidth}
    \centering
    \includegraphics[width=\linewidth]{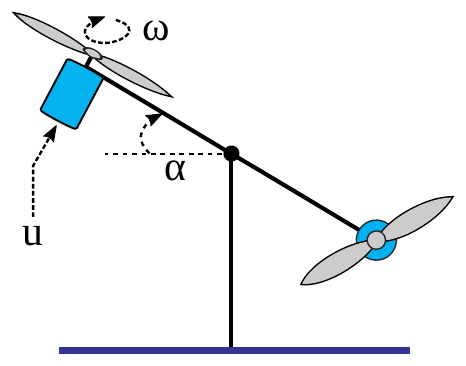}
    \caption{}
    \label{fig:heli-sub2}
  \end{subfigure}
  \caption{(a) The physical helicopter setup. (b) Schematic drawing of the helicopter setup.}
  \label{heli setup}
\end{figure}

The physical model of the helicopter, shown in \cref{heli model}, is a simplified model derived from first principles. Dissipative joint effects—such as friction and damping arising from arm rotations—are not included in the model. The system state is defined as $x = \{\omega, \alpha, \dot{\alpha}\}$,
where the angular velocity $\dot{\alpha}$ is unmeasurable and therefore treated as a latent variable.
The model contains five unknown physical parameters, 
$\theta = \{k_1,..., k_5\}$, and is expressed as
\begin{equation}
\begin{aligned}
\begin{bmatrix}
    \dot{\omega}\\
    \dot{\alpha}\\
    \ddot{\alpha}\\
    \end{bmatrix}=f_{phys}\left(\begin{bmatrix}
    \omega\\
    \alpha\\
    \dot{\alpha}\\
    \end{bmatrix},u \ \Bigg| \ \theta \right)&=
    \begin{bmatrix}
    k_2\omega+k_{3}\omega^2+k_1u\\
    \dot{\alpha}\\
    k_{4}\omega^2-k_5sin(\alpha)\\
    \end{bmatrix}.\\
    \label{heli model}\\[-0.8em] 
\end{aligned}
\end{equation}

\begin{figure*}[t]
    \centering
    \includegraphics[width=1.0\textwidth]{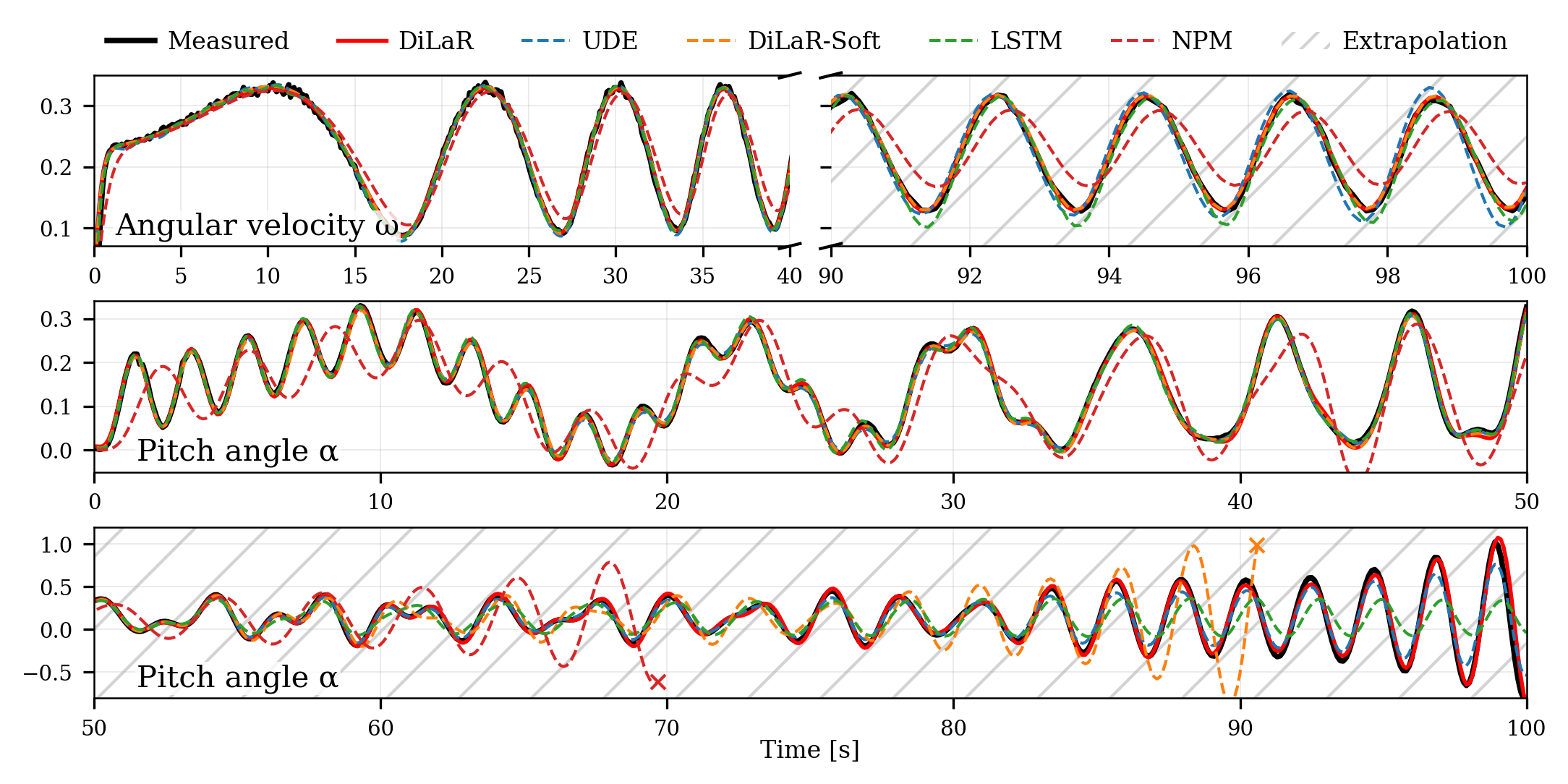}
    \vspace{-0.6cm}
    \caption{All models predict the system’s 100\,s output solely based on the given
    input $u$, their learned system dynamics, and the initial latent-state values. In Subfigure~3, the Nominal Physical Model (NPM) and DiLaR-Soft exhibit significant distortion at approximately $69\,\text{s}$ and $91\,\text{s}$, respectively, and their curves are therefore truncated. Overall, our model demonstrates the best performance, maintaining stable and close agreement with the measured data throughout the entire prediction horizon.}
    \label{heli fig}
\end{figure*}

According to fundamental physical principles, the total system energy can be expressed as the sum of
three decoupled components:
\[
V(x) = R_m(\omega) + P_h(\alpha) + R_h(\dot{\alpha}) \quad \text{with} \quad R_h = \tfrac{1}{2}I_h \dot{\alpha}^2,
\]
where each state contributes independently to the total energy. Specifically, $R_m$ denotes the motor’s rotational kinetic energy,
$P_h$ the helicopter’s gravitational potential energy,
and $R_h$ the helicopter’s rotational kinetic energy, with $I_h$ being the rotational inertia of the helicopter arm.
The derivative of the total energy with respect to the angular velocity $\dot{\alpha}$ is then
\[
\nabla_{\dot{\alpha}} V(x)
= d R_h(\dot{\alpha}) \big/ d \dot{\alpha}
= I_h \dot{\alpha},
\]
and, without loss of generality, $I_h$ is normalized to~1.

\subsection{Dataset and Training}

In the experiment, the input voltage $u$ is set as a chirp signal ranging from 0 to 0.4\,Hz for a duration of 100\,s, with a sampling interval of 0.1\,s (1000 samples). Only the first 500 samples are used for training, meaning that the model is exposed only to the system's low-frequency response but it must generalize to higher-frequency dynamics. During testing, the same 1000-point chirp signal $u$ is applied, and the model evolves solely according to its learned system dynamics and initial latent-state values. The first 500 predictions are referred to as the \textit{training region}, while the remaining 500 constitute the \textit{extrapolation region}.

Since the dimension of the latent state is $1$, the skew-symmetric matrix $S$ reduces to the constant zero, and the PSD matrix $K$ degenerates to a positive scalar. In DiLaR-PINN, the residual network models $K$ using a feedforward neural network with two hidden layers of 12 neurons each. The input and output dimensions are $4$ and $1$, respectively. All layers except the output layer use the $\tanh$ activation function, while the last layer uses \texttt{softplus} to ensure a positive output. During training, we first identify the parameters $\theta$ of the NPM using the \texttt{nlgreyest} function in MATLAB, and then use this identified $\theta$ as the initial value of the physical model part for UDE, DiLaR-SOFT, and DiLaR training.


\subsection{Experimental Results}
The experimental results shown in Fig.~\ref{heli fig} demonstrate that our method achieves the best prediction performance. 
Because the nominal physical model is an overly simplified representation of the system, NPM is almost incapable of accurately predicting the true outputs. 
As a black-box model, LSTM fits the training data well but fails to extrapolate to unseen high-frequency inputs due to the absence of any incorporated physical prior. 
The soft-constrained DiLaR-Soft model exhibits excellent fitting performance during training but gradually deviates from the true system behavior in the $\alpha$ state during the test phase, eventually leading to noticeable distortion. 
The UDE method---although the second-best performer overall---tends to learn an excessively dissipative dynamics. As shown in the third subfigure, from approximately $t = 90\,\text{s}$ onward, its predicted helicopter pitch angle becomes noticeably smaller than the measured values. 
In contrast, our method achieves accurate fitting within the training region and, more importantly, continues to predict the system outputs correctly when exposed to higher-frequency input signals. This demonstrates that the latent dissipative residual network successfully captures the unmodeled dissipative dynamics.

\begin{table}[t]
\centering
\caption{RMSE comparison among different baseline models.}
\renewcommand{\arraystretch}{1.15}
\setlength{\tabcolsep}{6pt}
\begin{tabular}{llccccc}
\toprule
 &  & 
 \raisebox{1ex}{\textbf{NPM}} &
 \raisebox{1ex}{\textbf{LSTM}} &
 \raisebox{1ex}{\textbf{UDE}} &
 \shortstack{\textbf{DiLaR}\\\textbf{(Soft)}} &
 \shortstack{\textbf{DiLaR}\\\textbf{(Ours)}} \\
\midrule
\multirow{3}{*}{\rotatebox{90}{\textbf{Train}}} 
 & $\omega$ & 0.0249 & \textbf{0.0048} & 0.0072 & 0.0066 & 0.0068 \\
 & $\alpha$ & 0.0635 & 0.0071 & 0.0063 & \textbf{0.0059} & \textbf{0.0053} \\
 & Overall & 0.0483 & \textbf{0.0060} & \textbf{0.0068} & \textbf{0.0062} & \textbf{0.0061} \\
\midrule
\multirow{3}{*}{\rotatebox{90}{\textbf{Test}}} 
 & $\omega$ & 0.0368 & 0.0108 & 0.0103 & \textbf{0.0046} & \textbf{0.0046} \\
 & $\alpha$ & 0.6682 & 0.1929 & 0.0704 & 0.4726 & \textbf{0.0504} \\
 & Overall & 0.4732 & 0.1367 & 0.0503 & 0.3342 & \textbf{0.0358} \\
\bottomrule
\label{heli table}
\end{tabular}
\end{table}

To quantitatively assess model performance, we computed the RMSE metrics, as presented in Table~\ref{heli table}. It can be seen that, except for the NPM, all methods achieve good fitting accuracy on the training data. However, on the test data, our method outperforms all baseline models. 
This improvement arises from the fact that dissipation is enforced through hard physical constraints, which leads to substantially more reliable generalization.

\section{Conclusion}
For electromechanical systems, dissipative effects such as joint friction and structural damping are among the most challenging components to model accurately. To address these modeling difficulties, we propose DiLaR-PINN, which combines a coarse physical model with a structured residual network. In particular, its residual component is constructed using a skew–dissipative parameterization that guarantees the residual network never injects energy into the system, thereby ensuring physically consistent modeling of dissipative dynamics.
Using a real-world helicopter setup, we demonstrate that DiLaR-PINN successfully learns the unmodeled joint-level dissipative effects, which are not trivial to capture, and outperforms all baseline approaches.
For future work, DiLaR-PINN may be applied to more complex systems—such as robotic manipulators and continuum robotic structures—to further explore its potential. 
In addition, the identified models could be exploited for controller learning and design, for example within reinforcement learning frameworks and sampling-based model predictive control schemes.



\bibliography{ifacconf}             
                                                   







\end{document}